\documentclass{IEEEtran}
\usepackage{cite}
\usepackage{amsmath,amssymb,amsfonts}
\usepackage{algorithmic}
\usepackage{graphicx}

\usepackage{mathtools}
\usepackage{amsfonts}
\usepackage{graphicx}
\usepackage[T1]{fontenc}
\usepackage{microtype}

\usepackage[standard,amsmath,thmmarks]{ntheorem}
\usepackage{dsfont}

\newcommand\reals{\mathbb{R}}

\newcommand\State{S}
\newcommand\Action{A}
\newcommand\state{s}

\newcommand{\SPSTATE}{\mathcal{\State}}
\newcommand\ACTION{\mathcal{\Action}}
\newcommand\action{a}
\newcommand\discount{\gamma}
\newcommand\policy{\pi}
\newcommand\polPars{\theta}
\newcommand\polParSpace{\Theta}
\newcommand\Reward{R}

\newcommand\Score{\Lambda}
\newcommand{\R}{\mathsf{R}}
\newcommand{\T}{\mathsf{T}}

\newcommand\EXP{\mathbb E}
\newcommand\PR{\mathbb P}
\newcommand\GRAD{\nabla}
\newcommand\EXPA{\EXP_{\Action_t\sim\policy(\State_t)}}
\newcommand\EXPB{\EXP_{\Action_t\sim\policy_\polPars(\State_t)}}

\newcommand{\IND}{\mathds{1}}

\newcommand\PRE[1]{#1^{-}}
\newcommand\POST[1]{#1^{+}}

\newcommand\PRESTATE{\PRE{\mathcal{\State}}}
\newcommand\POSTSTATE{\POST{\mathcal{\State}}}
\newcommand\Prestate{\PRE{\State}}
\newcommand\Poststate{\POST{\State}}
\newcommand\prestate{\PRE{s}}
\newcommand\poststate{\POST{s}}

\usepackage[vlined,ruled,norelsize,algo2e]{algorithm2e}
\usepackage{subcaption}
\SetKwInOut{Input}{input}
\SetKwInOut{Output}{output}
\SetKwRepeat{Do}{do}{while}

\usepackage{url}

\begin{document}
\title{Renewal Monte Carlo: \\Renewal theory based reinforcement learning}
\author{Jayakumar Subramanian %\IEEEmembership{Student Member, IEEE} 
and Aditya Mahajan %\IEEEmembership{Senior Member, IEEE}
\thanks{This work was supported by the Natural Sciences and Engineering Research Council of Canada through  NSERC Discovery Accelerator RGPAS 493011-16.}
\thanks{The authors are with the Electrical and Computer  Engineering
  Department, McGill University, Montreal, QC H3A\,0E9, Canada. (e-mails:
  jayakumar.subramanian@mail.mcgill.ca, aditya.mahajan@mcgill.ca)}}

\maketitle

\begin{abstract}
In this paper, we present an online reinforcement learning algorithm, called
Renewal Monte Carlo (RMC), for infinite horizon Markov decision processes with
a designated start state. RMC is a Monte Carlo algorithm and retains the
advantages of Monte Carlo methods including low bias, simplicity, and ease
of implementation while, at the same time, circumvents their key drawbacks of
high variance and delayed (end of episode) updates. The key ideas behind RMC
are as follows. First, under any reasonable policy, the reward process is
ergodic. So, by renewal theory, the performance of a policy is equal to the ratio
of expected discounted reward to the expected discounted time over a
regenerative cycle. Second, by carefully examining the expression for
performance gradient, we propose a stochastic approximation algorithm that
only requires estimates of the expected discounted reward and discounted time
over a regenerative cycle and their gradients. We propose two unbiased
estimators for evaluating performance gradients---a likelihood ratio based
estimator and a simultaneous perturbation based estimator---and show that for
both estimators, RMC converges to a locally optimal policy. We generalize the
RMC algorithm to post-decision state models and also present a variant that
converges faster to an approximately optimal policy. We conclude by
presenting numerical experiments on
a randomly generated MDP, event-triggered communication, and inventory
management. 
\end{abstract}


\section{Introduction}
\label{intro}

In recent years, reinforcement learning~\cite{BertsekasTsitsiklis1996,KaelblingLittmanMoore1996,sutton1998reinforcement,Szepesvari2009} has emerged as a leading framework to learn how to act optimally in unknown environments. Policy gradient methods~\cite{Sutton2000,Kakade2002,KondaTsitsiklis2003,
Schulman2015,Schulman2017,SilverSchrittwieserSimonyanEtAl2017} have played a prominent role in the success of reinforcement learning. Such methods have two critical components: policy evaluation and policy improvement. In policy evaluation step, the performance of a parameterized policy is evaluated while in the policy improvement step, the policy parameters are updated using stochastic gradient ascent.

Policy gradient methods may be broadly classified as Monte Carlo 
methods and temporal difference methods. In Monte Carlo methods, performance of a
policy is estimated using the discounted return of a single sample path; in
temporal difference methods, the value(-action) function is guessed and this guess is
iteratively improved using temporal differences. Monte Carlo methods are attractive
because they have zero bias, are simple and easy to implement, and work for
both discounted and average reward setups as well as for models with
continuous state and action spaces. However, they suffer from various drawbacks.
First, they have a high variance because a single sample path is used to
estimate performance. Second, they are not asymptotically optimal for
infinite horizon models because it is effectively
assumed that the model is episodic; in infinite horizon models, the
trajectory is arbitrarily truncated to treat the model as an episodic model.
Third, the policy improvement step cannot be carried out in tandem with
policy evaluation. One must wait until the end of the episode to estimate the
performance and only then can the policy parameters be updated. It is for
these reasons that Monte Carlo methods are largely ignored in the literature
on policy gradient methods, which
almost exclusively focuses on temporal difference methods such as actor-critic with eligibility traces~\cite{sutton1998reinforcement}. 

In this paper, we propose a Monte Carlo method---which we call \emph{Renewal
Monte Carlo} (RMC)---for infinite horizon Markov decision processes with
designated start state. Like Monte Carlo, RMC has low bias, is simple and easy
to implement, and works for models with continuous state and action spaces.
At the same time, it does not suffer from the drawbacks of typical Monte Carlo methods.
RMC is a low-variance online algorithm that works for infinite horizon
discounted and average reward setups. One doesn't have to wait until the end
of the episode to carry out the policy improvement step; it can be carried out
whenever the system visits the start state (or a neighborhood of
it).

Although renewal theory is commonly used to estimate performance of stochastic
systems in the simulation optimization community~\cite{Glynn1986,Glynn1990},
those methods assume that the probability law of the primitive random
variables and its weak derivate are known, which is not the case in
reinforcement learning. Renewal theory is also commonly used in
the engineering literature on queuing theory and systems and control for
Markov decision processes (MDPs) with average reward criteria and a known
system model. There is some prior work on using renewal theory for
reinforcement learning~\cite{MarbachTsitsiklis2001,MarbachTsitsiklis2003},
where renewal theory based estimators for the average return and differential
value function for average reward MDPs is developed. In RMC, renewal theory is
used in a different manner for discounted reward MDPs (and the results
generalize to average cost MDPs).

\section{RMC Algorithm} \label{sec:rl}

Consider a Markov decision process (MDP) with state $\State_t \in \SPSTATE$
and action $\Action_t \in \ACTION$. The system starts in an initial state
$\state_0 \in \SPSTATE$ and at time $t$:
\begin{enumerate}
  \item there is a controlled transition from $S_t$ to $S_{t+1}$ according to
    a transition kernel $P(\Action_t)$;
  \item a per-step reward $R_t =  r(\State_t, \Action_t, \State_{t+1})$ is
    received.
\end{enumerate}
Future is discounted at a rate $\discount \in (0,1)$. 

A (time-homogeneous and Markov) policy $\policy$ maps the current state to a
distribution on actions, i.e., $\Action_t \sim \policy(\State_t)$. We use
$\policy(\action | \state)$ to denote 
$\PR(\Action_t = \action | \State_t = \state)$.
The performance of a policy $\policy$ is given by
\begin{equation}
  J_\pi = 
  \EXPA\biggl[\sum_{t=0}^{\infty}\discount^{t}\Reward_t\biggm|\State_0 =
  \state_0\biggr]. \label{eq:Vp-defn}
\end{equation}

We are interested in identifying an optimal policy, i.e., a policy that maximizes the performance. When $\SPSTATE$ and $\ACTION$ are Borel spaces, we assume that the model satisfies the standard conditions under which time-homogeneous Markov policies are optimal~\cite{Hernandez-Lerma1996}.  In the
sequel, we present a sample path based online learning algorithm, which
we call Renewal Monte Carlo (RMC), which identifies a locally optimal policy
within the class of parameterized policies.

Suppose policies are parameterized by a closed and convex subset
$\polParSpace$ of the Euclidean space. For example, $\polParSpace$
  could be the weight vector in a Gibbs soft-max policy, or the weights
  of a deep neural network, or the thresholds in a control limit policy, and
so on. Given $\polPars \in \polParSpace$, we use $\policy_\polPars$ to denote
the policy parameterized by $\polPars$ and $J_\polPars$ to denote
$J_{\policy_{\polPars}}$. We assume that for all policies $\policy_\polPars$,
$\polPars \in \polParSpace$, the designated start state $s_0$ is positive
recurrent.

The typical approach for policy gradient based reinforcement learning is to start with an
initial guess $\polPars_0 \in \polParSpace$ and iteratively update it using
stochastic gradient ascent. In particular, let $\widehat \GRAD J_{\polPars_m}$ be an
unbiased estimator of $\GRAD_\polPars J_\polPars \big|_{\polPars =
\polPars_m}$, then update 
\begin{equation} \label{eq:J-update}
  \polPars_{m+1}
= \big[ \polPars_m + \alpha_m \widehat \GRAD J_{\polPars_m} \big]_{\polParSpace}
\end{equation} 
where $[\polPars]_{\polParSpace}$ denotes the projection of
$\polPars$ onto $\polParSpace$ and $\{\alpha_m\}_{m \ge 1}$ is the sequence of
learning rates that satisfies the standard assumptions of 
\begin{equation}\label{eq:lr}
  \sum_{m=1}^\infty \alpha_m = \infty 
  \quad\text{and}\quad 
  \sum_{m=1}^\infty \alpha_m^2 < \infty.
\end{equation}
Under mild
technical conditions~\cite{Borkar:book}, the above iteration converges to a $\polPars^*$ that is
locally optimal, i.e., $\GRAD_\polPars J_\polPars \big|_{\polPars =
\polPars^*} = 0$. In RMC, we approximate $\GRAD_\polPars J_\polPars$ by a
Renewal theory based estimator as explained below.

Let $\tau^{(n)}$ denote the stopping time when the system returns to the start
state $\state_0$ for the $n$-th time. In particular, let $\tau^{(0)}=0$  and
for $n \ge 1$  define 
\[ \tau^{(n)} = \inf\{t > \tau^{(n-1)}:\state_t = \state_0\}. \]
We call the sequence of $(\State_t, \Action_t,
\Reward_t)$ from $\tau^{(n-1)}$ to ${\tau^{(n)} - 1}$ as the
\emph{$n$-th regenerative cycle}. Let $\R^{(n)}$ and $\T^{(n)}$ denote the total
discounted reward and total discounted time of the $n$-th regenerative cycle,
i.e., 
\begin{align}\label{eq:Rn_and_Tn}
  \R^{(n)} = \Gamma^{(n)} 
  \smashoperator[r]{\sum_{t = \tau^{(n-1)}}^{\tau^{(n)} - 1}}
  \discount^{t} R_t
  \quad\text{and}\quad
  \T^{(n)} = \Gamma^{(n)}
  \smashoperator[r]{\sum_{t = \tau^{(n-1)}}^{\tau^{(n)}-1}}
  \discount^{t},
\end{align}
where $\Gamma^{(n)}=\discount^{-\tau^{(n-1)}}$.
By the strong Markov property, $\{\R^{(n)}\}_{n \ge 1}$ and $\{\T^{(n)}\}_{n \ge 1}$
are i.i.d.\@ sequences. Let $\R_\polPars$ and $\T_\polPars$ denote $\EXP[\R^{(n)}]$
and $\EXP[\T^{(n)}]$, respectively. Define
\begin{equation}
  \widehat \R = \frac 1N \sum_{n=1}^N \R^{(n)} 
  \quad \hbox{and}\quad
  \widehat \T = \frac 1N \sum_{n=1}^N \T^{(n)} ,
  \label{eq:est}
\end{equation}
where $N$ is a large number. 
Then, $\widehat \R$ and $\widehat \T$ are unbiased and asymptotically
consistent estimators of $\R_\polPars$ and $\T_\polPars$.

From ideas similar to standard
Renewal theory \cite{Feller1966}, we have the following.
\begin{proposition}[Renewal Relationship]\label{prop:renewal-basic1} The performance of policy $\policy_\polPars$ is given by:
\begin{equation}\label{eq:renewal-basic1}
  J_\polPars = \frac { \R_\polPars } { (1 - \discount) \T_\polPars }.
\end{equation}
\end{proposition}
\begin{proof}
  For ease of notation, define
  \[
    \overline \T_\polPars = \EXPB\big[ \discount^{\tau^{(n)} - \tau^{(n-1)}}
    \big]
  \]
  Using the formula for geometric series, we get that $\T_\polPars = ( 1 -
  \overline \T_\polPars)/(1 - \discount)$. Hence,
  \begin{equation} \label{eq:Tbar1}
    \overline \T_\polPars = 1 - (1 - \discount) \T_\polPars.
  \end{equation}

  Now, consider the performance:
  \begin{align}
    J_\polPars &= \EXPB\bigg[
      \sum_{t=0}^{\tau^{(1)}-1} \discount^{t} R_t 
      \notag \\
      & \hspace{7em}
    +  \discount^{\tau^{(1)}}
    \smashoperator{\sum_{t = \tau^{(1)}}^\infty} \discount^{t-\tau^{(1)}} R_t 
    \biggm| \State_{0} = \state_0 \bigg]
    \displaybreak[0]
    \notag \\
    &\stackrel{(a)}= \R_\polPars + 
    \EXPB[ \discount^{\tau^{(1)})} ]\, J_\polPars
    \displaybreak[0]
    \notag \\
    &= \R_\polPars + \overline \T_\polPars J_\polPars,
    \label{eq:R-11}
  \end{align}
  where the second expression in $(a)$ uses the independence of random
  variables from $(0, \tau^{(1)}-1)$ to those from $\tau^{(1)}$ onwards 
  due to the strong Markov property. Substituting~\eqref{eq:Tbar1}
  in~\eqref{eq:R-11} and rearranging terms, we get the result of the
  proposition.
\end{proof}
Differentiating both sides of Equation~\eqref{eq:renewal-basic1} with respect to $\polPars$, we get that
\begin{equation} \label{eq:H}
  \GRAD_\polPars J_\polPars = \frac{H_\polPars}{\T_\polPars^2(1 - \discount)},
  \enskip\text{where }
  H_\polPars =  \T_\polPars \GRAD_\polPars \R_\polPars
  - \R_\polPars \GRAD_\polPars \T_\polPars. 
\end{equation}

Therefore, instead of using  stochastic gradient ascent to find the maximum
of $J_\polPars$, we can use stochastic approximation to find
the root of $H_\polPars$. In particular, let $\widehat H_m$ be an unbiased
estimator of $H_{\polPars_m}$. We then use the update
\begin{equation} \label{eq:H-update}
  \polPars_{m+1} = \big[ \polPars_m + \alpha_m \widehat H_m \big]_{\polParSpace}
\end{equation}
where $\{\alpha_m\}_{m \ge 1}$ satisfies the standard conditions on learning
rates~\eqref{eq:lr}.  The above iteration converges to a locally optimal policy.
Specifically, we have the following.

\begin{theorem}\label{thm:convergence}
  Let $\widehat \R_m$, $\widehat \T_m$, $\widehat \GRAD \R_m$ and $\widehat
  \GRAD \T_m$ be unbiased estimators of $\R_{\polPars_m}$, $\T_{\polPars_m}$,
  $\GRAD_\polPars \R_{\polPars_m}$, and $\GRAD_\polPars \R_{\polPars_m}$,
  respectively such that $\widehat \T_m \perp \widehat \GRAD \R_m$ and
  $\widehat \R_m \perp \widehat \GRAD \T_m$.\footnote{The notation $X \perp Y$
  means that the random variables $X$ and $Y$ are independent.} Then,
  \begin{equation}
    \widehat H_m = \widehat \T_m \widehat \GRAD \R_m - \widehat \R_m 
    \widehat \GRAD \T_m
    \label{eq:H-est}
  \end{equation}
  is an unbiased estimator of $H_\polPars$ and 
  the sequence $\{\polPars_m\}_{m \ge 1}$ generated
  by~\eqref{eq:H-update} converges almost surely and 
  \[
    \lim_{m \to \infty} \GRAD_\polPars J_\polPars \big|_{\polPars_m} = 0.
  \]
\end{theorem}
\begin{proof}
  The unbiasedness of $\widehat H_m$ follows immediately from the
  independence assumption. The convergence of
  the $\{\polPars_m\}_{m \ge 1}$ follows from~\cite[Theorem 2.2]{Borkar:book}
  and the fact that the model satisfies conditions (A1)--(A4)
  of~\cite[pg~10--11]{Borkar:book}.
\end{proof}

In the remainder of this section, we present two methods for estimating the
gradients of $\R_\polPars$ and $\T_\polPars$. 
The first is a likelihood ratio based
gradient estimator which works when the policy is differentiable with respect
to the policy parameters. The second is a simultaneous perturbation based
gradient estimator that uses finite differences, which is useful when the
policy is not differentiable with respect to the policy parameters.

\subsection{Likelihood ratio based gradient based estimator}\label{sec:likelihood}

One approach to estimate the performance gradient is to use likelihood radio
based estimates~\cite{Rubinstein1989,Glynn1990,Williams1992}.
Suppose the policy $\policy_\polPars(\action | \state)$ is differentiable with respect to
$\polPars$. For any time~$t$, define the likelihood function
\begin{equation}\label{eq:score}
  \Score_t = 
  \GRAD_\polPars \log[ \policy_\polPars(\Action_t \mid \State_t) ],
\end{equation}
and for $\sigma \in \{ \tau^{(n-1)}, \dots, \tau^{(n)} - 1 \}$, define
\begin{equation}
  \R^{(n)}_\sigma =  \Gamma^{(n)}\sum_{t=\sigma}^{\tau^{(n)}-1}\discount^tR_t,\enskip
  \T^{(n)}_\sigma = \Gamma^{(n)}\sum_{t=\sigma}^{\tau^{(n)}-1}\discount^t.\label{eq:R_T_sigma}
\end{equation}
In this notation $\R^{(n)} = \R^{(n)}_{\tau^{(n-1)}}$ and $\T^{(n)} = \T^{(n)}_{\tau^{(n-1)}}$.
Then, define the following estimators for $\GRAD_\polPars \R_\polPars$ and
$\GRAD_\polPars \T_\polPars$:
\begin{align}
\widehat \GRAD \R &= \frac 1N \sum_{n=1}^N \sum_{\sigma=\tau^{(n-1)}}^{\tau^{(n)}-1}\R^{(n)}_\sigma  \Score_{\sigma},\label{eq:grad_R_new}\\
\widehat \GRAD \T &= \frac 1N \sum_{n=1}^N \sum_{\sigma=\tau^{(n-1)}}^{\tau^{(n)}-1}\T^{(n)}_\sigma  \Score_{\sigma},\label{eq:grad_T_new}
\end{align}
where $N$ is a large number. 

\begin{proposition} \label{prop:estimator}
  $\widehat \GRAD \R$ and $\widehat \GRAD \T$ defined above are unbiased
  and asymptotically consistent estimators
  of $\GRAD_\polPars \R_\polPars$ and $\GRAD_\polPars \T_\polPars$. 
\end{proposition}
\begin{proof}
Let $P_\polPars$ denote the probability induced on the sample paths when the
system is following policy $\policy_\polPars$. For $t \in \{ \tau^{(n-1)},
\dots, \tau^{(n)} - 1\}$, let $D^{(n)}_{t}$ denote the sample path $(\State_s,
\Action_s, \State_{s+1})_{s=\tau^{(n-1)}}^{t}$ for
the $n$-th regenerative cycle until time $t$. Then,
\[
  \let\smashoperator\relax
  P_\polPars(D^{(n)}_t) = \smashoperator{\prod_{s= \tau^{(n-1)}}^{t} }
  \policy_\polPars(A_s | S_s)
  \PR(\State_{s+1} | S_{s}, A_{s})
\]
Therefore,
\begin{equation}  \label{rel:11}
  \GRAD_\polPars \log P_\polPars(D^{(n)}_t) =
  \smashoperator{\sum_{s=\tau^{(n-1)}}^{t}} \GRAD_\polPars \log
  \policy_\polPars(\Action_s | \State_s) = \smashoperator{\sum_{s=\tau^{(n-1)}}^{t}} \Score_s.
\end{equation}

Note that $\R_\polPars$ can be written as:
\[
  \R_\polPars = \Gamma^{(n)}\smashoperator[r]{\sum_{t = \tau^{(n-1)}}^{\tau^{(n)} - 1}}\discount^{t}\EXPB[ R_t].
\]
Using the log derivative trick,\footnote{Log-derivative trick: For any distribution $p(x|\theta)$ and any function $f$,
\[
  \GRAD_\theta \EXP_{X \sim p(X|\theta)} [ f(X) ]
  = 
  \EXP_{X \sim p(X|\theta)}[ f(X) \GRAD_\theta \log p(X | \theta)].
\]
} we get
\begin{align} 
  \GRAD_\polPars \R_\polPars &= 
  \Gamma^{(n)} 
    \smashoperator[r]{\sum_{t = \tau^{(n-1)}}^{\tau^{(n)} - 1}}
    \discount^{t}\,
  \EXPB[  R_t \GRAD_\polPars \log P_\polPars(D^{(n)}_t) ] 
  \notag \\
  &\stackrel{(a)}= 
  \Gamma^{(n)} 
  \EXPB\bigg[
  \sum_{t = \tau^{(n-1)}}^{\tau^{(n)} - 1}
  \bigg[
\discount^{t} R_t{\sum_{\sigma=\tau^{(n-1)}}^t}\Score_\sigma\bigg] \bigg]
  \notag \\
    &\stackrel{(b)}= \EXPB\bigg[ 
  \sum_{\sigma = \tau^{(n-1)}}^{\tau^{(n)} - 1}\Score_\sigma\bigg[
  \Gamma^{(n)}\sum_{t=\sigma}^{\tau^{(n)} - 1}\discount^{t}  R_t \bigg] 
  \bigg] \notag \\
      &\stackrel{(c)}= \EXPB\bigg[ 
  \sum_{\sigma = \tau^{(n-1)}}^{\tau^{(n)} - 1}
  \R^{(n)}_\sigma \Score_\sigma \bigg]
  \label{rel:12}
\end{align}
where $(a)$ follows from~\eqref{rel:11}, $(b)$ follows from changing the order
of summations, and $(c)$ follows from the definition of
$\R^{(n)}_\sigma$ in~\eqref{eq:R_T_sigma}. $\widehat \GRAD
\R$ is an unbiased and asymptotically consistent estimator of the right hand
side of the first equation in~\eqref{rel:12}. The result for $\widehat \GRAD
\T$ follows from a similar argument.
\end{proof}

\begin{algorithm2e}[!tb]
\def\1#1{\quitvmode\hbox to 1em{\hfill$\mathsurround0pt #1$}}
\SetKwInOut{Input}{input}
\SetKwInOut{Output}{output}
\SetKwInOut{Init}{initialize}
\SetKwProg{Fn}{function}{}{}
\SetKwFor{ForAll}{forall}{do}{}
\SetKwRepeat{Do}{do}{while}
\DontPrintSemicolon
\Input{Intial policy $\polPars_0$, discount factor $\discount$, 
  initial state~$\state_0$, number of regenerative cycles $N$}

\For{iteration $m = 0, 1, \dots$}{
  \For{regenerative cycle $n_1=1$ to $N$}{
  Generate $n_1$-th regenerative cycle using
  \rlap{policy~$\policy_{\polPars_m}$.}

  Compute $\R^{(n_1)}$ and $\T^{(n_1)}$ using~\eqref{eq:Rn_and_Tn}.
  }
  Set $\widehat \R_{m} = \texttt{average}(\R^{(n_1)}: n_1 \in \{1,\dots,N\})$.

  Set $\widehat \T_{m} = \texttt{average}(\T^{(n_1)}: n_1 \in \{1,\dots,N\})$.

  \For{regenerative cycle $n_2=1$ to $N$}{
    Generate $n_2$-th regenerative cycle using
    \rlap{policy~$\policy_{\polPars_m}$.}

    Compute $\R_\sigma^{(n_2)}$, $\T_\sigma^{(n_2)}$ and $\Score_\sigma$
    for all $\sigma$.
  }
   Compute $\widehat \GRAD \R_{m}$ and $\widehat \GRAD \T_m$ 
   using \eqref{eq:grad_R_new} and~\eqref{eq:grad_T_new}.

  \vskip 2pt
   Set $\widehat H_m = \widehat \T_m \widehat \GRAD \R_m - \widehat \R_m \widehat \GRAD \T_m$.

  \vskip 4pt
  Update $\polPars_{m+1} = \big[ \polPars_m + \alpha_m \widehat H_m
  \big]_{\polParSpace}$.
}
\caption{RMC Algorithm with likelihood ratio based gradient estimates.}
\label{alg:likelihood}
\end{algorithm2e}

To satisfy the independence condition of Theorem~\ref{thm:convergence}, we use two independent sample paths: one to estimate $\widehat \R$ and
$\widehat \T$ and the other to estimate $\widehat \GRAD \R$ and $\widehat
\GRAD \T$. The complete algorithm in shown in Algorithm~\ref{alg:likelihood}.
An immediate consequence of Theorem~\ref{thm:convergence} is the following. 
\begin{corollary}\label{cor:pol_grad_conv}
  The sequence $\{\polPars_m\}_{m \ge 1}$ generated by
  Algorithm~\ref{alg:likelihood} converges to a local optimal.
\end{corollary}

\begin{remark}\label{rem:1}
Algorithm~\ref{alg:likelihood} is presented in its simplest form. It is possible to use standard variance reduction techniques such as
subtracting a baseline~\cite{Williams1992,Greensmith2004,Peters2006} to reduce variance. 
\end{remark}

\begin{remark}\label{rem:2} \label{rem:single_run}
In Algorithm~\ref{alg:likelihood}, we use two separate runs to compute $(\widehat \R_m, \widehat \T_m)$ and $(\GRAD \widehat \R_m, \GRAD \widehat \T_m)$ to ensure that the independence conditions of Proposition~\ref{prop:estimator} are satisfied. In practice, we found that using a single run to compute both $(\widehat \R_m, \widehat \T_m)$ and $(\GRAD \widehat \R_m, \GRAD \widehat \T_m)$ has negligible effect on the accuracy of convergence (but speeds up convergence by a factor of two).
\end{remark}

\begin{remark}\label{rem:3}
It has been reported in the literature~\cite{Thomas2014} that using a biased estimate of the gradient given by:
\begin{equation}
  \R^{(n)}_\sigma = \Gamma^{(n)} 
  \sum_{t=\sigma}^{\tau^{(n)}-1}\discount^{t-\sigma} R_t,
  \label{eq:R_sigma_biased} 
\end{equation}
(and a similar expression for $T^{(n)}_\sigma$) leads to faster convergence. We call this variant \textit{RMC with biased gradients} and, in our experiments, found that it does converge faster than RMC.
\end{remark}

\subsection{Simultaneous perturbation based gradient estimator}
Another approach to estimate performance gradient is to use simultaneous
perturbation based estimates\cite{Spall1992,Maryak2008,Katkovnik1972,Bhatnagar:2013}. The general one-sided form of such estimates is
\[
  \widehat \GRAD \R_\polPars = \delta (
  \widehat \R_{\polPars + c \delta} - \widehat \R_{\polPars} )/c
\]
where $\delta$ is a random variable with the same dimension as
$\polPars$ and $c$ is a small constant. The expression for $\widehat \GRAD
\T_\polPars$ is similar. When $\delta_i \sim 
\text{Rademacher}(\pm 1)$, the above method corresponds
to simultaneous perturbation stochastic approximation (SPSA)~\cite{Spall1992,
Maryak2008}; when $\delta \sim \text{Normal}(0, I)$, the above method
corresponds to smoothed function stochastic approximation
(SFSA)~\cite{Katkovnik1972,Bhatnagar:2013}.

\begin{algorithm2e}[!tb]
\def\1#1{\quitvmode\hbox to 1em{\hfill$\mathsurround0pt #1$}}
\SetKwInOut{Input}{input}
\SetKwInOut{Output}{output}
\SetKwInOut{Init}{initialize}
\SetKwProg{Fn}{function}{}{}
\SetKwFor{ForAll}{forall}{do}{}
\SetKwRepeat{Do}{do}{while}
\DontPrintSemicolon
\Input{Intial policy $\polPars_0$, discount factor $\discount$, initial state~$\state_0$, number of regenerative cycles $N$, constant $c$,  perturbation distribution
$\Delta$}

\For{iteration $m = 0, 1, \dots$}{
  \For{regenerative cycle $n_1=1$ to $N$}{
    Generate $n_1$-th regenerative cycle using
    \rlap{policy~$\policy_{\polPars_m}$.}

    Compute $\R^{(n_1)}$ and $\T^{(n_1)}$ using~\eqref{eq:Rn_and_Tn}.
  }
  Set $\widehat \R_{m} = \texttt{average}(\R^{(n_1)}: n_1 \in \{1,\dots,N\})$.

  Set $\widehat \T_{m} = \texttt{average}(\T^{(n_1)}: n_1 \in \{1,\dots,N\})$.

  Sample $\delta \sim \Delta$.

  Set $\polPars_m' = \polPars_m + c \delta$.

  \For{regenerative cycle $n_2=1$ to $N$}{
    Generate $n_2$-th regenerative cycle using
    \rlap{policy~$\policy_{\polPars_m}$.}

  Compute $\R^{(n_2)}$ and $\T^{(n_2)}$ using~\eqref{eq:Rn_and_Tn}.
  }
  Set $\widehat \R'_{m} = \texttt{average}(\R^{(n_2)}: n_2 \in \{1,\dots,N\})$.

  Set $\widehat \T'_{m} = \texttt{average}(\T^{(n_2)}: n_2 \in \{1,\dots,N\})$.

  \vskip 2pt
  Set $\widehat H_m = \delta(\widehat \T_m \widehat \R'_m 
  - \widehat \R_m \widehat \T'_m)/c$.

  \vskip 4pt
  Update $\polPars_{m+1} = \big[ \polPars_m + \alpha_m \widehat H_m
  \big]_{\polParSpace}$.
}
\caption{RMC Algorithm with simultaneous perturbation based gradient estimates.}
\label{alg:SPSA}
\end{algorithm2e}

Substituting the above estimates in~\eqref{eq:H-est} and simplifying, we get
\[
  \widehat H_\polPars = \delta ( \widehat \T_\polPars \widehat \R_{\polPars + c\delta} 
    - 
  \widehat \R_\polPars \widehat \T_{\polPars + c \delta} )/c.
\]
The complete algorithm in shown in Algorithm~\ref{alg:SPSA}. Since $(\widehat
\R_\polPars, \widehat \T_\polPars)$ and $(\widehat \R_{\polPars + c \delta},
\widehat \T_{\polPars + c \delta})$ are estimated from separate sample paths,
$\widehat H_\polPars$ defined above is an unbiased estimator of $H_\polPars$.
Then, an immediate consequence of Theorem~\ref{thm:convergence} is the
following.
\begin{corollary}
  The sequence $\{\polPars_m\}_{m \ge 1}$ generated by
  Algorithm~\ref{alg:SPSA} converges to a local optimal.
\end{corollary}

\section{RMC for Post-Decision State Model} \label{sec:post_model}
In many models, the state dynamics can be split into two parts: a controlled evolution followed by an uncontrolled evolution. For example, many continuous state models have dynamics of the form
\[
S_{t+1} = f(S_t, A_t) + N_t,
\]
where $\{N_t\}_{t \ge 0}$ is an independent noise process. For other examples, see the inventory control and event-triggered communication models in Sec~\ref{sec:num_exp}. Such models can be written in terms of a post-decision state model described below.

Consider a post-decision state MDP with pre-decision state $\Prestate_t \in
\PRESTATE$, post-decision state $\Poststate_t \in \POSTSTATE$, action
$\Action_t \in \ACTION$.
The system starts at an initial state $\poststate_0 \in \POSTSTATE$ and at 
time~$t$: 
\begin{enumerate}
  \item there is a controlled transition from $\Prestate_t$ to $\Poststate_t$
    according to a transition kernel $\PRE P(\Action_t)$; 
  \item there is an uncontrolled transition from $\Poststate_t$ to
    $\Prestate_{t+1}$ according to a transition kernel $\POST P$;
  \item a per-step reward $R_t =  r(\Prestate_t, \Action_t,
    \Poststate_t)$ is received.
\end{enumerate}
Future is discounted at a rate $\discount \in (0,1)$. 
\begin{remark}
  When $\POSTSTATE = \PRESTATE$ and $\PRE P$ is identity, then the above model
  reduces to the standard MDP model, considered in Sec~\ref{sec:rl}. When
  $\POST P$ is a deterministic transition, the model reduces to a standard MDP
  model with post decision
  states~\cite{VanRoyBertsekasLeeEtAl1997,powell2011approximate}. 
\end{remark}

As in Sec~\ref{sec:rl}, we choose a (time-homogeneous and Markov) policy $\policy$ that maps the current pre-decision state $\PRESTATE$ to a
distribution on actions, i.e., $\Action_t \sim \policy(\Prestate_t)$. We use
$\policy(\action | \prestate)$ to denote 
$\PR(\Action_t = \action | \Prestate_t = \prestate)$. 

The performance when the system starts in post-decision state $\poststate_0 \in
\POSTSTATE$ and follows policy $\policy$ is given by
\begin{equation}
  J_\policy =  
  \EXPA\biggl[\sum_{t=0}^{\infty}\discount^{t}\Reward_t\biggm|\Poststate_0 =
  \poststate_0\biggr]. %\label{eq:Vp-defn}
\end{equation}
As before, we are interested in identifying an optimal policy, i.e., a policy that maximizes the performance. When $\SPSTATE$ and $\ACTION$ are Borel spaces, we assume that the model satisfies the standard conditions under which time-homogeneous Markov policies are optimal~\cite{Hernandez-Lerma1996}.
Let $\tau^{(n)}$ denote the stopping times such that $\tau^{(0)} = 0$ and for
$n \ge 1$,
\[
  \tau^{(n)} = \inf \{ t > \tau^{(n-1)} : \poststate_{t-1} = \poststate_0 \}.
\]
The slightly unusual definition (using $\poststate_{t-1} = \poststate_0$
rather than the more natural $\poststate_t = \poststate_0$) is to ensure that
the formulas for $\R^{(n)}$ and $\T^{(n)}$ used in Sec.~\ref{sec:rl} remain
valid for the post-decision state model as well. Thus, using arguments similar to Sec.~\ref{sec:rl}, we can show that both variants of RMC
presented in Sec.~\ref{sec:rl} converge to a locally optimal parameter
$\polPars$ for the post-decision state model as well.

\section{Approximate RMC}\label{sec:approx_rl}

In this section, we present an approximate version of RMC (for the basic
model of Sec.~\ref{sec:rl}). Suppose that the state and action spaces
$\SPSTATE$ and $\ACTION$ are separable metric spaces (with metrics $d_S$ and
$d_A$).

Given an approximation constant $\rho \in \reals_{> 0}$, let $B^\rho = \{s \in \SPSTATE: d_S(s,s_0) \le \rho\}$ denote
the ball of radius $\rho$ centered around $s_0$. Given a policy $\policy$, let
$\tau^{(n)}$ denote the stopping times for successive visits to $B^\rho$,
i.e., $\tau^{(0)} = 0$ and for $n \ge 1$, 
\[
  \tau^{(n)} = \inf \{ t > \tau^{(n-1)} : \state_t \in B^\rho \}.
\]
Define $\R^{(n)}$ and $\T^{(n)}$ as in~\eqref{eq:Rn_and_Tn} and let
$\R^\rho_\polPars$ and $\T^\rho_\polPars$ denote the expected values of
$\R^{(n)}$ and $\T^{(n)}$, respectively. Define
\[
  J^\rho_\polPars = \frac{\R^\rho_\polPars}{ (1-\discount) \T^\rho_\polPars}.
\]

\begin{theorem}\label{thm:approx_RMC}
  Given a policy $\policy_\polPars$, let $V_\polPars$ denote the value
  function and $\overline \T^\rho_\polPars = \EXPB[ \discount^{\tau^{(1)}} |
  \State_0 = \state_0 ]$ (which is always less than $\discount$). Suppose the
  following condition is satisfied:
  \begin{enumerate}
    \item[\textup{(C)}] The value function $V_\polPars$ is locally Lipschitz
      in $B^\rho$, i.e., there exists a $L_\polPars$ such that for any $s, s'
      \in B^\rho$, 
      \[
        | V_\polPars(s) - V_\polPars(s') | \le L_\polPars d_S(s,s').
      \]
  \end{enumerate}
  Then
  \begin{equation}\label{eq:approxJ_bound}
    \big| J_\polPars - J^\rho_\polPars \big| \le 
    \frac{ L_\polPars \overline \T^\rho_\polPars } { (1-\discount)
    \T^\rho_\polPars} \rho \le \frac{\discount}{(1-\discount)} L_\polPars \rho.
  \end{equation}
\end{theorem}
\begin{proof}
  We follow an argument similar to Proposition~\ref{prop:renewal-basic1}.
  \begin{align}
    J_\polPars &= V_\polPars(s_0) = 
    \EXPB\bigg[
      \sum_{t=0}^{\tau^{(1)}-1} \discount^{t} R_t 
      \notag \\
      & \hskip 6em 
    +  \discount^{\tau^{(1)}}
    \sum_{t = \tau^{(1)}}^\infty \discount^{t-\tau^{(1)}} R_t 
    \biggm| \State_{0} = \state_{\tau^{(1)}} \bigg]
    \notag \\
    &\stackrel{(a)}= \R^\rho_\polPars + 
\EXPB[ \discount^{\tau^{(1)}} | \State_0 = \state_0]\, V_\polPars(\state_{\tau^{(1)}})
\label{eq:approx1}
  \end{align}
  where $(a)$ uses the strong Markov property.
  Since $V_\polPars$ is locally Lipschitz with constant $L_\polPars$ and
  $s_{\tau^{(1)}} \in B^\rho$, we have that
  \[
    |J_\polPars - V_\polPars(s_{\tau^{(1)}}) | = |V_\polPars(s_0) -
    V_\polPars(s_{\tau^{(1)}}) | \le L_\polPars \rho. 
  \]
  Substituting the above in~\eqref{eq:approx1} gives
  \[
    J_\polPars \le \R^\rho_\polPars + \overline \T^\rho_\polPars (J_\polPars +
    L_\polPars \rho).
  \]
  Substituting $\T^\rho_\polPars = (1 - \overline \T^\rho_\polPars)/(1 -
  \discount)$ and rearranging the terms, we get 
  \[
    J_\polPars \le J^\rho_\polPars + \frac{L_\polPars \overline
    \T^\rho_\polPars}{(1-\discount) \T^\rho_\polPars } \rho.
  \]
  The other direction can also be proved using a similar argument. The second
  inequality in~\eqref{eq:approxJ_bound} follows from $\overline \T^\rho_\polPars \le \gamma$ and ${\T^\rho_\polPars \ge 1}$.
\end{proof}

Theorem~\ref{thm:approx_RMC} implies that we can find an approximately optimal policy by identifying policy parameters $\polPars$ that minimize $J^\rho_\polPars$. To do so, we can appropriately modify both variants of
RMC defined in Sec.~\ref{sec:rl} to declare a renewal whenever the state lies
in $B^\rho$. 

For specific models, it may be possible to verify that the value function is
locally Lipschitz (see Sec.~\ref{sec:inv_ctrl} for an example). However, we
are not aware of general conditions that guarantee local Lipschitz
continuity of value functions. It is possible to identify sufficient conditions that guarantee global Lipschitz continuity of value functions (see~\cite[Theorem 4.1]{Hinderer2005},
\cite[Lemma 1, Theorem 1]{Rachelson2010}, \cite[Lemma 1]{Pirotta2015}). We state these conditions below.
\begin{proposition}\label{prop:Lispschitz}
  Let $V_\polPars$ denote the value function for any policy
  $\policy_{\polPars}$. Suppose the model satisfies the following conditions:
  \begin{enumerate}
    \item The transition kernel $P$ is Lipschitz, i.e., there
      exists a constant $L_P$ such that for all
      $s,s' \in \SPSTATE$ and $a,a' \in \ACTION$, 
      \[
        \mathcal K(P(\cdot | s,a), P(\cdot | s',a')) \le 
        L_P\big[ d_S(s,s') + d_A(a,a') \big],
      \]
      where $\mathcal K$ is the Kantorovich metric (also called
      Kantorovich-Monge-Rubinstein metric or Wasserstein distance) between
      probability measures.

    \item The per-step reward $r$ is Lipschitz, i.e., there exists a constant
      $L_r$ such that for all $s,s',s_+ \in \SPSTATE$ and $a,a' \in \ACTION$,
      \[
        | r(s,a,s_+) - r(s',a',s_+) | \le 
          L_r\big[ d_S(s,s') + d_A(a,a') \big].
      \]
  \end{enumerate}
  In addition, suppose the policy satisfies the following:
  \begin{enumerate}
    \setcounter{enumi}{2}
    \item The policy $\policy_\polPars$ is Lipschitz, i.e., there exists a
      constant $L_{\policy_\polPars}$ such that for any $s,s' \in \SPSTATE$,
      \[
        \mathcal K( \policy_\polPars(\cdot | s), \policy_\polPars(\cdot | s')) 
        \le
        L_{\policy_\polPars}\, d_S(s,s').
      \]
    \item $\discount L_P(1 + L_{\policy_\polPars}) < 1$.
    \item The value function $V_\polPars$ exists and is finite. 
  \end{enumerate}
  Then, $V_\polPars$ is Lipschitz. In particular, for any $s, s' \in
  \SPSTATE$,
  \[
    | V_\polPars(s) - V_\polPars(s') | \le L_\polPars d_S(s,s'),
  \]
  where
  \[
    L_\polPars = \frac{L_r (1 + L_{\policy_\polPars})}
    {1 - \discount L_P(1 + L_{\policy_\polPars}) }.
  \]
\end{proposition}

\section{Numerical Experiments}\label{sec:num_exp}

We conduct three experiments to evaluate the performance of RMC: a randomly
generated MDP, event-triggered communication, and inventory
management.  

\subsection{Randomized MDP (GARNET)} \label{sec:GARNET}

In this experiment, we study a randomly generated $\text{GARNET}(100,10,50)$ 
model~\cite{Bhatnagar2009}, which is an MDP with $100$ states, $10$ actions,
and a branching factor of $50$ (which means that each row of all transition
  matrices has $50$ non-zero elements, chosen $\text{Unif}[0,1]$ and
normalized to add to~$1$). For each state-action pair, with probability
$p=0.05$, the reward is chosen $\text{Unif}[10,100]$, and with probability
$1-p$, the reward is~$0$. Future is discounted by a factor of
$\discount=0.9$. The first state is chosen as start state. The policy is a Gibbs soft-max distribution
parameterized by $100 \times 10$ (states $\times$ actions) parameters, where
each parameter belongs to the interval $[-30, 30]$. The temperature of the
Gibbs distribution is kept constant and equal to~$1$.

We compare the performance of RMC, RMC with biased gradient (denoted by
RMC-B, see Remark~\ref{rem:2}), and actor critic with eligibility
traces for the critic~\cite{sutton1998reinforcement} (which we refer to as
SARSA-$\lambda$ and abbreviate as S-$\lambda$ in the plots),
with $\lambda \in \{0, 0.25, 0.5, 0.75, 1\}$.
 For both the RMC algorithms, we use the same runs to estimate the gradients
 (see Remark~\ref{rem:single_run} in Sec.~\ref{sec:rl}).
\def\PARAMS{For
all algorithms, the learning rate is chosen using ADAM~\cite{ADAM} with
default hyper-parameters and the $\alpha$ parameter of ADAM equal to $0.05$
for RMC, RMC-B, and the actor in SARSA-$\lambda$ and the learning rate is equal to $0.1$ for the
critic in SARSA-$\lambda$. For RMC and RMC-B, the policy parameters are
updated after $N=5$ renewals.}
Each algorithm\footnote{\PARAMS} is run $100$ times and the mean and standard deviation of the
performance (as estimated by the algorithms themselves) is shown in
Fig.~\ref{fig:GARNET-RL-train}. The performance of the corresponding policy
evaluated by Monte-Carlo evaluation over a horizon of $250$~steps and averaged
over $100$~runs is shown in Fig.~\ref{fig:GARNET-RL-eval}. The optimal
performance computed using value iteration is also shown.

The results show that SARSA-$\lambda$ learns faster (this is expected because
the critic is keeping track of the entire value function) but has higher
variance and gets stuck in a local minima. On the other hand, RMC and RMC-B
learn slower but have a low bias and do not get stuck in a local minima. The
same qualitative behavior was observed for other randomly generated models. Policy gradient algorithms only guarantee convergence to a local optimum. We are not sure why RMC and SARSA differ in which local minima they
converge to. Also, it was observed that RMC-B (which is RMC with biased evaluation of the gradient)
learns faster than RMC.

\begin{figure}[!t!b]
  \centering
    \begin{subfigure}{1.0\linewidth}
    \centering
    \includegraphics[width=\linewidth]{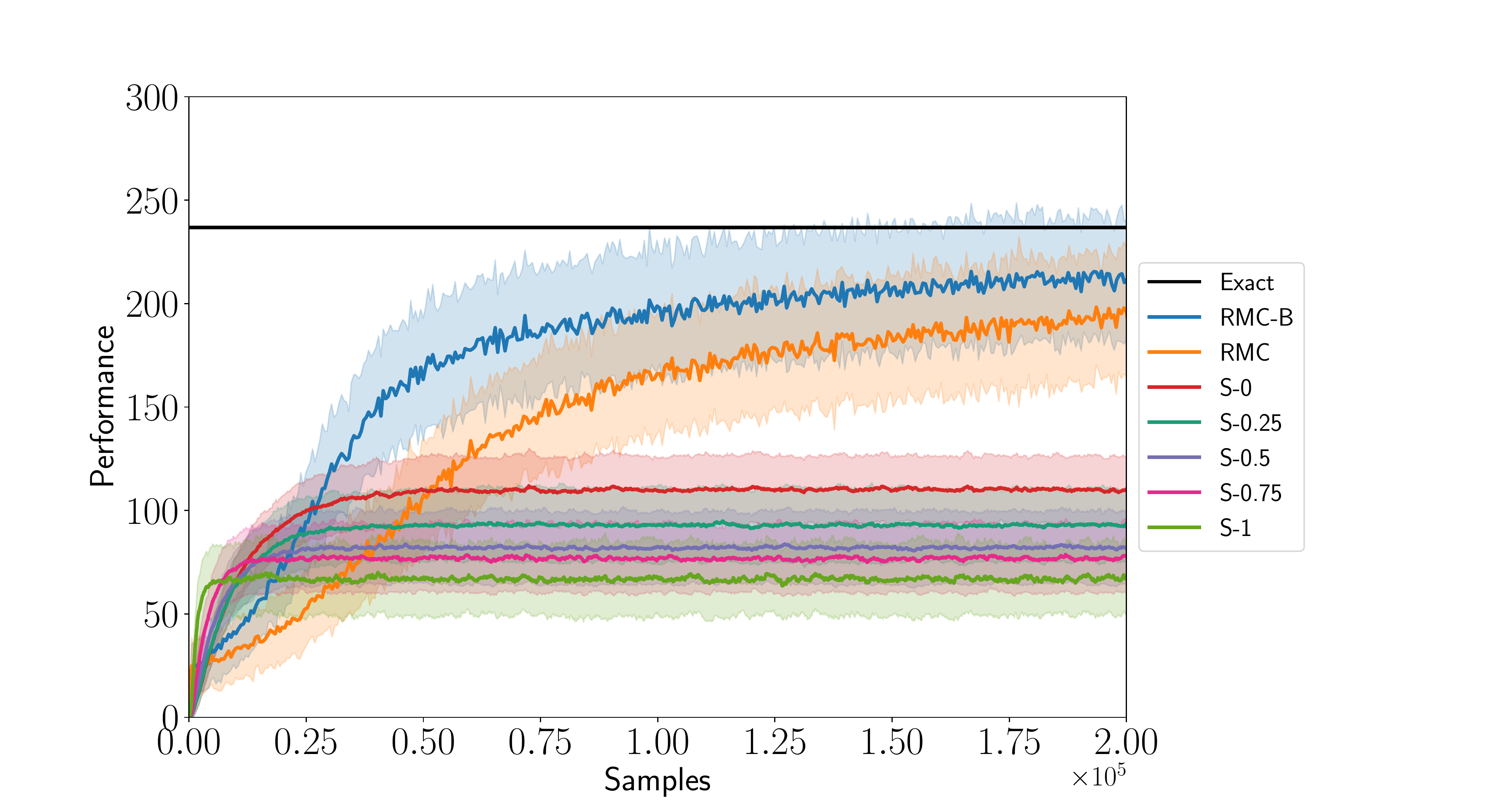}
    \caption{}
    \label{fig:GARNET-RL-train}
  \end{subfigure}
  \begin{subfigure}{1.0\linewidth}
    \centering
    \includegraphics[width=\linewidth]{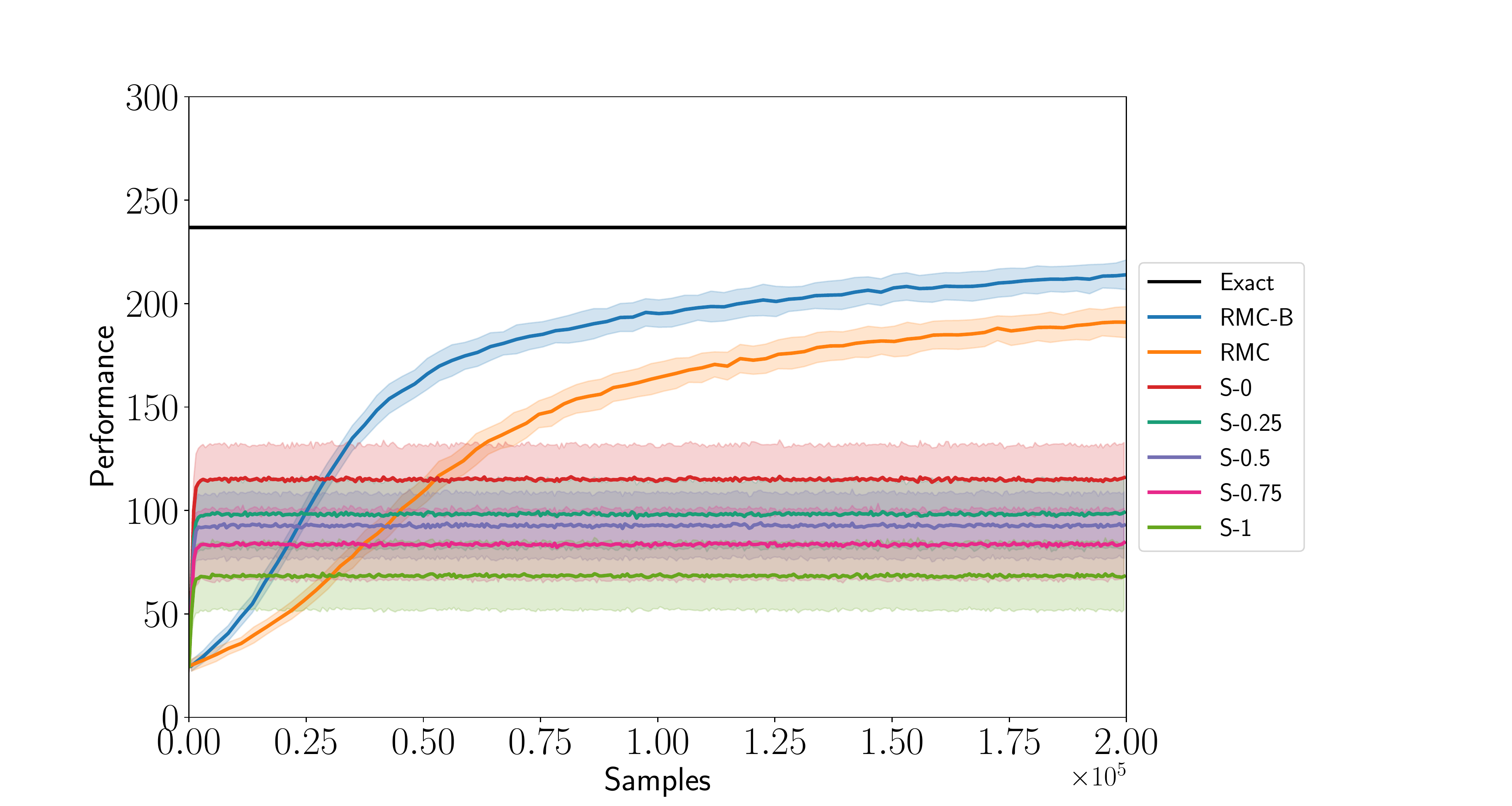}
    \caption{}
    \label{fig:GARNET-RL-eval}
  \end{subfigure}
    \caption{Performance of different learning algorithms on
      $\text{GARNET}(100,10,50)$ with $p=0.05$ and $\discount=0.9$. 
      (a)~The performance estimated by the algorithms online. (b)~The
      performance estimated by averaging over $100$ Monte Carlo evaluations
    for a rollout horizon of $250$. The solid lines show the mean value and the shaded region shows the $\pm$ one standard deviation region.}
\end{figure}

\subsection{Event-Triggered Communication} \label{sec:rem_est}

\begin{figure}[!t!b]
  \centering
  \renewcommand\unitlength{cm}
  \includegraphics[width=1.0\linewidth]{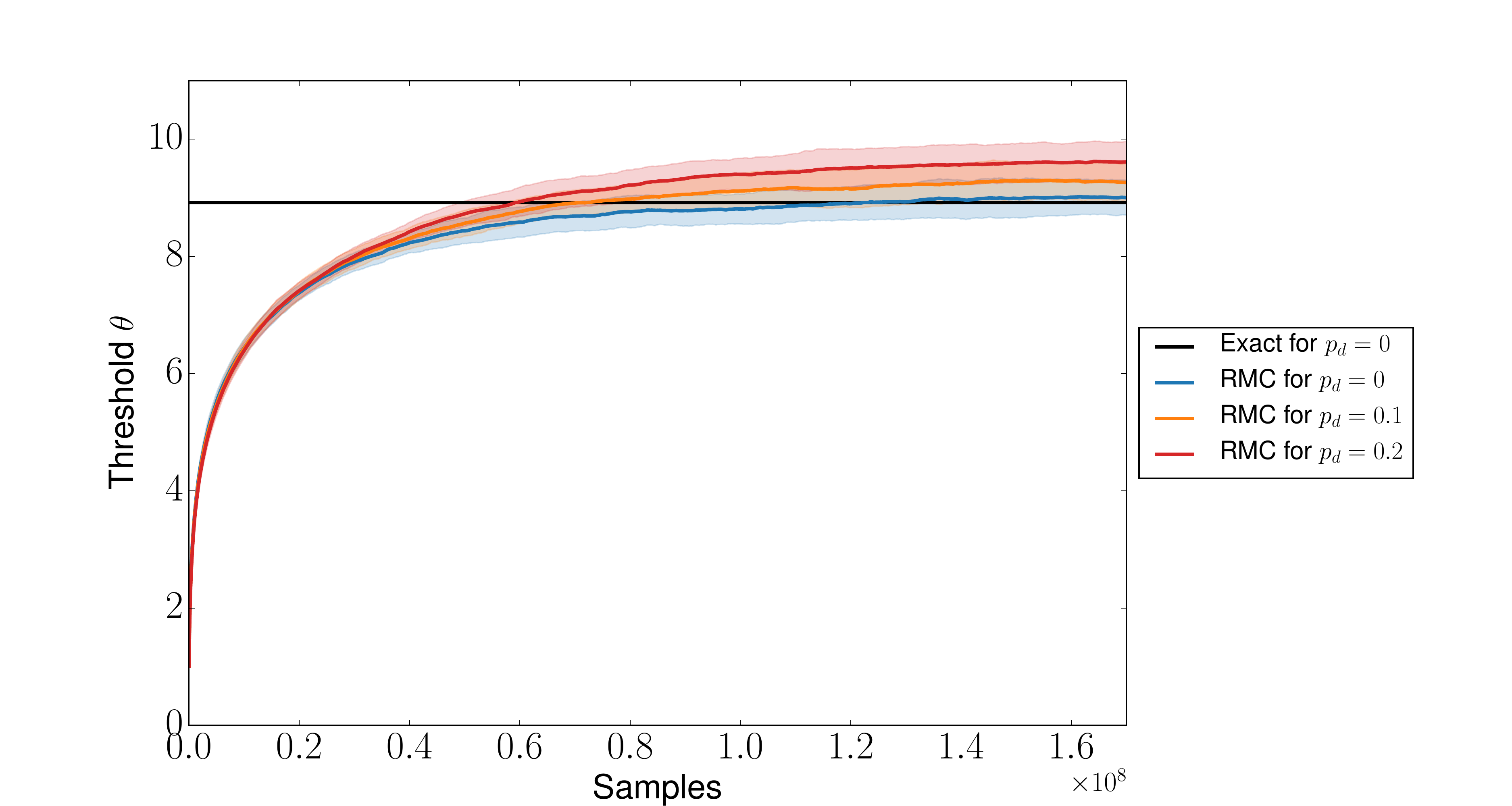}
  \caption{Policy parameters versus number of samples (sample values averaged over 100 runs) for event-driven
    communication using RMC for different values of $p_d$. The solid lines
  show the mean value and the shaded area shows the $\pm$ one standard
  deviation region.}
  \label{fig:RE}
\end{figure}

In this experiment, we study an event-triggered communication problem that arises in networked control systems~\cite{LipsaMartins:2011,CSM:thresholds}. A transmitter observes a first-order autoregressive process $\{X_t\}_{t \ge 1}$, i.e., $X_{t+1} =
\alpha X_t + W_t$, where $\alpha, X_t, W_t \in \reals$, and $\{W_t\}_{t \ge 1}$ is
an i.i.d.\ process. At each time, the transmitter uses an event-triggered
policy (explained below) to determine whether to transmit or not (denoted by
$A_t = 1$ and $A_t = 0$, respectively). Transmission takes place over an
i.i.d.\ erasure channel with erasure probability $p_d$. Let $\Prestate_t$ and
$\Poststate_t$ denote the ``error'' between the source realization and it's
reconstruction at a receiver. It can be shown that $\Prestate_t$ and
$\Poststate_t$ evolve as follows~\cite{LipsaMartins:2011,CSM:thresholds}: when $A_t = 0$,
$\Poststate_t = \Prestate_t$; when $A_t = 1$, $\Poststate_t = 0$ if the
transmission is successful (w.p. $(1-p_d)$) and $\Poststate_t = \Prestate_t$ if
the transmission is not successful (w.p. $p_d$); 
and $\Prestate_{t+1} = \alpha \Poststate_t + W_t$. Note that this is a post-decision state model, where the post-decision
state resets to zero after every successful transmission.\footnote{Had we used
the standard MDP model instead of the post-decision state model,
this restart would not have always resulted in a renewal.}%

The per-step cost has two components: a communication cost of
$\lambda A_t$, where $\lambda \in \reals_{> 0}$ and an estimation error 
$(\Poststate_t)^2$. The objective is to minimize the expected discounted cost. 

An event-triggered policy is a threshold policy that chooses $A_t = 1$
whenever $|\Prestate_t| \ge \polPars$, where $\polPars$ is a design choice.
Under certain conditions, such an event-triggered policy is known to be
optimal~\cite{LipsaMartins:2011,CSM:thresholds}. When the system model is known, algorithms to compute the optimal $\polPars$ are presented
in~\cite{XuHes2004,CM:remote-estimation}. In this section, we use RMC to
identify the optimal policy when the model parameters are not known. 

In our experiment we consider an event-triggered model with $\alpha = 1$,
$\lambda = 500$, $p_d \in \{0, 0.1, 0.2\}$, $W_t \sim {\cal N}(0, 1)$,
$\discount = 0.9$, and use simultaneous perturbation variant of RMC\footnote{An
event-triggered policy is a parametric policy but $\policy_\polPars(\action |
\prestate)$ is not differentiable in $\polPars$. Therefore, the likelihood
ratio method cannot be used to estimate performance gradient.} to identify
$\polPars$. We run the algorithm 100 times and the result for different
choices of $p_d$ are shown in Fig.~\ref{fig:RE}.\footnote{We choose the
learning rate using ADAM with default hyper-parameters and the $\alpha$
parameter of ADAM equal to 0.01. We choose $c = 0.3$, $N=100$ and $\Delta =
\mathcal{N}(0,1)$ in Algorithm~\ref{alg:SPSA}.} For $p_d = 0$, the optimal
threshold computed using~\cite{CM:remote-estimation} is also shown.
The results show that RMC converges relatively quickly and has low bias across multiple runs.

\subsection{Inventory Control} \label{sec:inv_ctrl}

In this experiment, we study an inventory management problem that arises in operations
research~\cite{Arrow1951,Bellman1955}. Let $S_t \in \subset \reals$ denote
the volume of goods stored in a warehouse, $A_t \in \reals_{\ge 0}$ denote the
amount of goods ordered, and $D_t$ denotes the demand. The state evolves
according to $S_{t+1} = S_t + A_t - D_{t+1}$. 

We work with the normalized cost function:
\[C(s) = a_p s (1-\discount)/\discount + a_h s \IND_{\{ s \ge 0\}} - a_b s \IND_{\{s < 0\}}, \] 
where $a_p$ is the procurement cost, $a_h$ is the holding cost, and $a_b$ is the backlog cost (see~\cite[Chapter 13]{Whittle1982}
for details).

It is known that there exists a threshold $\theta$ such that the optimal
policy is a base stock policy with threshold $\theta$ (i.e., whenever the
current stock level falls below $\theta$, one orders up to $\theta$).
Furthermore, for $s \le \theta$, we have that~\cite[Sec~13.2]{Whittle1982}
\begin{equation}\label{eq:opt-IC}
  V_\polPars(s) = C(s) + \frac{\discount}{(1-\discount)} \EXP[C(\polPars - D) ].
\end{equation}
So, for $B^\rho \subset (0, \theta)$, the value function is locally Lipschitz, with
\[
  L_\polPars = \left( a_h + \frac{1 - \discount} {\discount} a_p \right).
\]
So, we can use
approximate RMC to learn the optimal policy.

In our experiments, we consider an inventory management model with $a_h = 1$, $a_b
= 1$, $a_p = 1.5$, $D_t \sim \text{Exp}(\lambda)$ with $\lambda = 0.025$, start
state $s_0 = 1$, discount factor $\discount = 0.9$, and use
simultaneous perturbation variant of approximate RMC to identify $\theta$. 
We
run the algorithm $100$ times and the result is shown in
Fig.~\ref{fig:inv_ctl-RL}.\footnote{We choose the learning rate using ADAM
with default hyper-parameters and the $\alpha$ parameter of ADAM equal to
0.25. We choose $c = 3.0$, $N=100$, and $\Delta = \mathcal{N}(0,1)$ in
Algorithm~\ref{alg:SPSA} and choose $\rho = 0.5$ for approximate RMC\@. We bound the states within $[-100.0, 100.0]$.} The
optimal threshold and performance computed using~\cite[Sec 13.2]{Whittle1982}%
\footnote{For $\text{Exp}(\lambda)$ demand,
  the optimal threshold is (see~\cite[Sec 13.2]{Whittle1982})
  \[\polPars^* = \frac 1\lambda 
    \log\left( \frac{a_h + a_b}{a_h + a_p(1-\gamma)/\gamma)} \right).\]
}
is also shown. 
The result shows that RMC converges to an approximately optimal parameter value with total cost within the bound predicted in Theorem~\ref{thm:approx_RMC}.

\begin{figure}[!t!b]
  \centering
    \begin{subfigure}{1.0\linewidth}
    \centering
    \includegraphics[width=\linewidth]{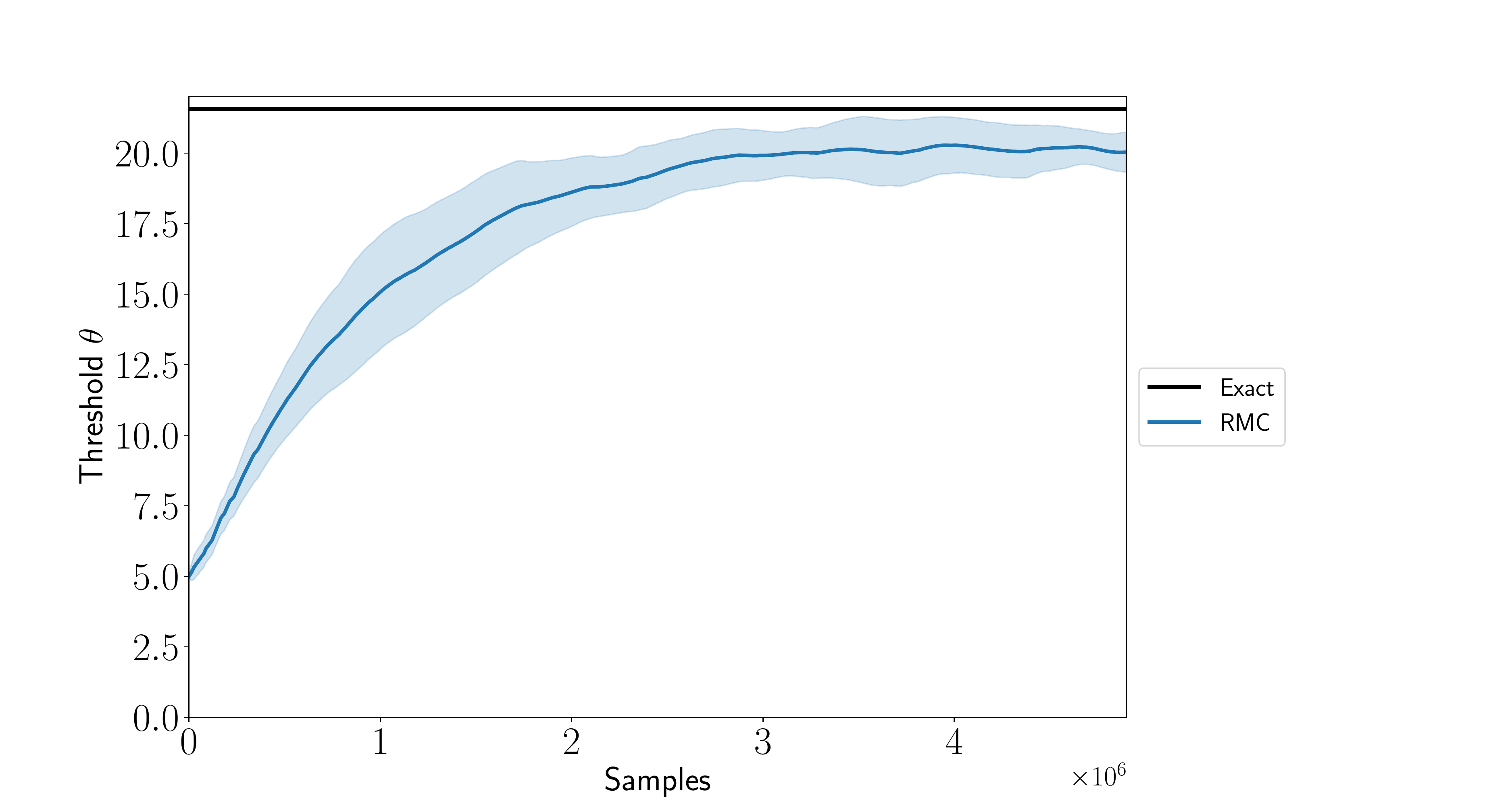}
    \caption{}
    \label{fig:inv_ctl-RL_threshold}
  \end{subfigure}
  \begin{subfigure}{1.0\linewidth}
    \centering
    \includegraphics[width=\linewidth]{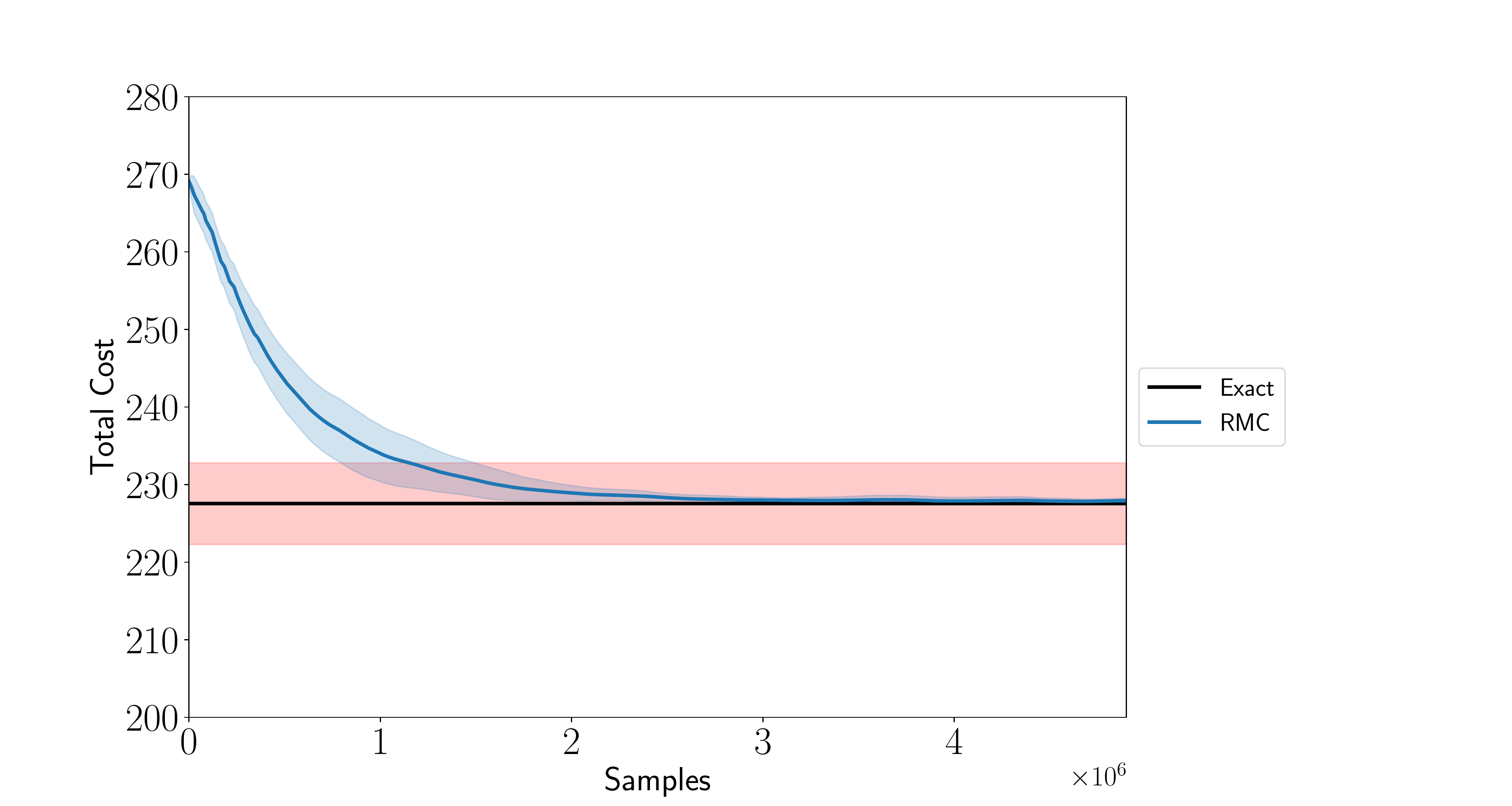}
    \caption{}
    \label{fig:inv_ctl-RL_perf}
  \end{subfigure}
    \caption{(a) Policy parameters and (b) Performance (total cost) versus
      number of samples (sample values averaged over 100 runs) for inventory
      control using RMC\@. The solid lines show the mean value and the shaded area shows the $\pm$ one standard
deviation region. In (b), the performance is computed using~\eqref{eq:opt-IC} for the policy parameters given in (a). The red rectangular region shows the total cost bound given by Theorem~\ref{thm:approx_RMC}.}
\label{fig:inv_ctl-RL}
\end{figure}

\section{Conclusions}

We present a renewal theory based reinforcement learning algorithm called
Renewal Monte Carlo. RMC retains the key advantages of Monte Carlo methods and
has low bias, is simple and easy to implement, and works for models with
continuous state and action spaces. In addition, due to the averaging over
multiple renewals, RMC has low variance. We generalized the
RMC algorithm to post-decision state models and also presented a variant that converges faster to an approximately optimal policy, where the renewal state is replaced by a renewal set. The error in using such an approximation is bounded by the size of the renewal set.

In certain models, one is interested in the peformance at a reference state
that is not the start state. In such models, we can start with an arbitrary
policy and ignore the trajectory until the reference state is visited for
the first time and use RMC from that time onwards (assuming that the reference
state is the new start state).

The results presented in this paper also apply to average reward models where
the objective is to maximize
\begin{equation}
  J_\pi = 
  \lim_{t_h \to
  \infty}\frac{1}{t_h}\EXPA\biggl[\sum_{t=0}^{t_h-1}\Reward_t\biggm|\State_0 =
  \state_0\biggr]. \label{eq:avg_Vp-defn}
\end{equation}
Let the stopping times $\tau^{(n)}$ be defined as before. Define the total reward
$\R^{(n)}$ and duration $\T^{(n)}$ of the $n$-th regenerative cycle as
\[
  \R^{(n)} = 
  \smashoperator[r]{\sum_{t = \tau^{(n-1)}}^{\tau^{(n)} - 1}} R_t
  \quad\text{and}\quad
  \T^{(n)} =  \tau^{(n)} - \tau^{(n-1)}.
\]
Let $\R_\polPars$ and $\T_\polPars$ denote the expected values of $\R^{(n)}$
and $\T^{(n)}$ under policy $\policy_{\polPars}$. Then from standard renewal
theory we have that the performance $J_\polPars$ is equal to 
$\R_\polPars/ \T_\polPars$ and, therefore $\GRAD_\polPars J_\polPars =
H_\polPars/T^2_\polPars$, where $H_\polPars$ is defined as in~\eqref{eq:H}. We
can use both variants of RMC prosented in Sec.~\ref{sec:rl} to obtain
estimates of $H_\polPars$ and use these to update the policy parameters
using~\eqref{eq:H-update}.

\section*{Acknowledgment}

The authors are grateful to Joelle Pineau for useful feedback and for suggesting the idea of approximate RMC.

\bibliographystyle{IEEEtran}
\bibliography{IEEEabrv,rmc_tac}

\end{document}